\documentclass{article}

\usepackage{microtype}
\usepackage{graphicx}
\usepackage{amsmath}
\usepackage{amsthm}
\usepackage{multirow}
\newtheorem{theorem}{Theorem}[section]
\usepackage{amssymb}
\graphicspath{{./Images/}}
\usepackage{subfigure}
\usepackage{booktabs} 

\usepackage{hyperref}

\usepackage[accepted]{icml2020}

\icmltitlerunning{Fairness Preferences, Actual and Hypothetical: A Study of Crowdworker Incentives}

\begin{document}

\twocolumn[
\icmltitle{Fairness Preferences, Actual and Hypothetical: A Study of Crowdworker Incentives}




\begin{icmlauthorlist}
\icmlauthor{Angie Peng}{to}
\icmlauthor{Jeff Naecker}{to}
\icmlauthor{Ben Hutchinson}{to}
\icmlauthor{Andrew Smart}{to}
\icmlauthor{Nyalleng Moorosi}{to}
\end{icmlauthorlist}

\icmlaffiliation{to}{Google AI, Mountain View, CA, USA}

\icmlcorrespondingauthor{Angie Peng}{apeng@google.com}

\icmlkeywords{Machine Learning, ICML}

\vskip 0.3in
]



\printAffiliationsAndNotice{\icmlEqualContribution} 

\begin{abstract}
How should we decide which fairness criteria or definitions to adopt in machine learning systems? To answer this question, we must study the fairness preferences of actual users of machine learning systems. Stringent parity constraints on treatment or impact can come with trade-offs, and may not even be preferred by the social groups in question \cite{zafar2017parity}. Thus it might be beneficial to elicit what the group’s preferences are, rather than rely on {\em a priori} defined mathematical fairness constraints. Simply asking for self-reported rankings of users is challenging because research has shown that there are often gaps between people’s stated and actual preferences\cite{bernheim2013non}.

This paper outlines a research program and experimental designs for investigating these questions. Participants in the experiments are invited to perform a set of tasks in exchange for a base payment---they are told upfront that they may receive a bonus later on, and the bonus could depend on some combination of output quantity and quality. The same group of workers then votes on a bonus payment structure, to elicit preferences. The voting is hypothetical (not tied to an outcome) for half the group and actual (tied to the actual payment outcome) for the other half, so that we can understand the relation between a group’s actual preferences and hypothetical (stated) preferences. Connections and lessons from fairness in machine learning are explored.

\end{abstract}

\section{Introduction}
\label{section:Introduction}
The increasingly large-scale deployments of AI systems on global populations require large quantities of resources and coordination \cite{crawford2018anatomy, amodei31ai}, which are often achieved through centralized and top-down decision-making processes.
In doing so, the diversity of stakeholders and their perspectives are often given short shrift.
Indeed, even when deployment objectives include moral values such as fairness, we often see first principles-based approaches, which aim to proceed from abstract ideals.  
However, stringent top-down constraints on treatment or impact can come with trade-offs, and these constraints may not even be preferred by the groups most affected by the system. \cite{zafar2017parity}.

What if instead, the fairness properties of a machine learning system were engineered to better incorporate the preferences of its users?
Answering this question is important because qualitative work has shown that perceptions of algorithmic unfairness among technology users negatively affects their trust in a company or product \cite{woodruff2018qualitative}.
We argue that to design ML for stakeholders, we must 1) understand dimensions of variation of their preferences and 2) design social choice mechanisms for aggregating diverse individual preferences to reach a collective decision \cite{arrow1951social}.
While there have been some studies seeking to understand people's preferences regarding ML models \cite{grgic2018human, harrison2020empirical, binns2018s, jahanbakhsh2017you,  saxena2019, srivastava2019mathematical},
most previous work on eliciting judgments around fairness have typically focused on qualitative perceptions \cite{woodruff2018qualitative}, hypothetical scenarios \cite{lee2018understanding}, or actual scenarios in a hypothetical setting \cite{harrison2020empirical}. 
In contrast, our goal is to measure fairness preferences for the same situation in \emph{both} a hypothetical and actual setting in order to measure the gap between stated and actual preferences. Questions of preferences are challenging because research has shown that there are often gaps between people’s stated and actual preferences\cite{samuelson1948consumption}.
Quantifying the relationship between stated and real preferences then allows us to quickly and more efficiently make predictions about actual preferences from survey data.

Studying preferences in this area is particularly challenging for another reason: research has showing that people can reverse their preferences depending on the context, violating some common modeling assumptions, such as time consistency of preferences.
In the specific context of considering ``fairness'', \citet{andreoni2018} have shown that whether the preference between equality of opportunity ({\em ax ante} fairness) and equality of outcomes ({\em ex post} fairness) may not even be stable within individuals.
Thus when eliciting preferences in this area, we must be careful to consider how luck and randomness play a role.\textbf{}

In this paper, we present an experimental design which aims to elicit preferences from an often-overlooked stakeholder group: the crowdworkers who annotate ML data sets \cite{gray2019ghost,ross2010crowdworkers}. Given the structural correspondences between pay incentive structures and ML systems (which we'll describe in Section \ref{section:fairpay}), these experiments also serve as a useful proxy for discussing perspectives on fairness in ML more generally.

This paper is structured as follows. We first provide a brief overview of the structural relationships between fairness in the context of crowdworker incentives and fairness theorems in machine learning. These relationships motivate the design of the experiment which we proceed to outline in Section \ref{section:experiment}, including ethical considerations in the design. We then describe how the experimental results will be used to shed new light on relationships between stated and actual preferences around ML model deployments, and conclude with a discussion of a larger research agenda that could build from these results.

\section{Fair Pay/Work Criteria}
\label{section:fairpay}
The crowdworkers who label the datasets powering ML systems are a critical part of the ML economy, and as such their employment conditions are deserving of consideration \cite{gray2019ghost,ross2010crowdworkers}. Concerns about crowdworkers are furthermore situated within a much bigger societal conversation about how fair and equitable pay might be determined for different types of work. One pertinent question is how the aggregate value paid to a workforce should be divided between the individual workers.\footnote{This paper does not consider the question of {\em how much} of the value of the collective work should be returned to the workers.}
In this paper we focus on the relationship of pay to the measurable durations and outputs of work. 
That is, for the sake of this paper, we suppose that the fairness of pay could be linked to the inputs of the labor process, such as time and intensity spent on work, as well as the quality and quantity of the outputs. 
Notions of {\em fair pay} under consideration may appeal to any or all of these quantities. 
In some cases, input and output might be highly correlated, but in other cases difference in technology -- or simply luck -- may disrupt this relationship.


\begin{table}[ht]
    \centering
    \small
    \renewcommand{\arraystretch}{1.5}
    \begin{tabular}{p{2.8cm}p{2cm}p{2.5cm}}
    \toprule
    \bf Pay criterion & \bf ML criterion & \bf ML terminology \\
    \midrule
    $Group \perp Pay$
    & $Group \perp \widehat{Y}$ & Demographic parity 
    \\\hline
    $Group \perp Pay | \widehat{labor}$
    & $Group \perp \widehat{Y} | Y$ & Equal opportunity
\\\hline
    $Group \perp  \widehat{labor} | Pay $
    & $ Group \perp Y | \widehat{Y}$ & Sufficiency 
    \\\bottomrule
    \end{tabular}
    \caption{Comparison of a selection of possible fairness criteria in the ML and pay domains. ``Hat''-notation ($\hat{\cdot}$) denotes the estimate of an actual value. Labor can refer to a combination of input and output variables (e.g. complexity of problems given and quantity of tasks completed).
}
    \label{tab:analogs}

\end{table}

A direct parallel can be made between fairness in pay and machine learning fairness. In particular, Table~\ref{tab:analogs} outlines some structural relationships between different possible criteria of fairness in the domains of pay and machine learning. We claim that these definitions model several important dimensions of fairness of pay/work, such as ``is it unfair if I am paid less than someone who does the same work as me?'' (second line of the table), and ``is it unfair if I do more work than someone who is paid the same as me?'' (third line of the table). 

One complicating factor is that, in many scenarios, labor can be hard or impossible to measure accurately.
We assert that this actually directly parallels the situation in many machine learning tasks, in which the appropriate sociotechnical framing of fairness is not just complex and contextual  \cite{selbst2019fairness}, but for which the question of construct measurement is critical \cite{jacobs2019measurement, jacobs2020meaning}.

With these considerations in mind, it is instructive to reconsider the so-called ``group fairness impossibility'' theorems of recent years \cite{pleiss2017fairness, kleinberg2016inherent, chouldechova2017fair}.
In the fair pay scenario, we have the following directly analogous result.

\begin{theorem}[Fair Pay Impossibility Result]
 Suppose that pay is calculated as a function of the measured outputs of work, but that these outputs are not a perfect proxy for labor. Furthermore suppose that different groups have different distributions of pay and/or labor. Then it is impossible for the following to both hold: $Group \perp Pay | labor$ and $Group \perp labor | Pay$. Similarly, it is impossible  for the following to both hold: $Pay \perp labor | Outputs$ and $Pay \perp Outputs | labor$.
\end{theorem}

\begin{proof} Follows from direct consideration of the fairness in ML impossibility results in \citet{kleinberg2016inherent}.
\end{proof}

Where does this leave us? \citet{darlington1971another} made two relevant observations when considering a similar scenario regarding university admissions. First, the more accurate that our measurement of outputs is, the ``closer'' we can get to achieving the incompatible conditions listed in above. (And conversely, if measurements are extremely noisy then one independence condition or the other---or both!---must be violated to a large degree.) Second, one possible resolution that Darlington suggests is to poll people in order to determine socially acceptable trade offs. In the following section, we describe an experimental design aimed at precisely this.

\section{Proposed Experiment}
\label{section:experiment}

\subsection{Overview}

Participants are invited to perform a set of tasks in exchange for a base payment - they are told upfront that they may receive a bonus later on, and the bonus could depend on some combination of output quantity and quality. In our particular experiment, we have set up a standard transcription task for participants to perform; the task itself could theoretically be anything, so long as there is something that results in work that varies by quality and quantity. The goal of having participants complete a task is to then give the same group the ability to vote on a bonus payment structure. It is important that the group voting has actually done some work, so that we can understand the group’s actual preferences instead of hypothetical (stated) preferences

\subsection{Specific Task}
Our specific task involves transcription short images of handwritten lines through Amazon Mechanical Turk. Participants are given one hour to complete the following:
\begin{itemize}
    \item Answer questions about their demographic background 
    \item Read through instructions and do a few practice transcription tasks 
    \item Transcribe as many handwritten lines to the best of their ability. See an example transcription task below  (Figure \ref{fig:transcription}).
\end{itemize}

\begin{figure}[ht]
    \centering
    \includegraphics[width=8cm]{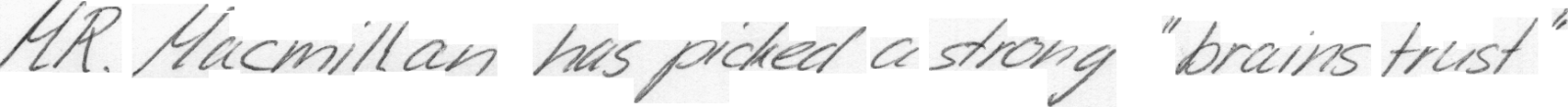}
    \caption{An example transcription task (from IAM handwriting database)}
    \label{fig:transcription}
\end{figure}

\subsection{Data Collection}
Initially, our experiment will run in the United States, though one of our goals for expansion is to do the same experiment in India to get a sense of differences by country. Our plan is to collect 1,000 complete tasks and ask participants to self-report information on the following demographics: age group, geographic region (within the United States), and race/ethnicity. Given that mTurk is not representative of broader populations in the US, we want to use the demographics collected to provide a comparison point against broader population distributions.

\subsection{Voting on a Bonus Payment}
Of the 1,000 participants, we propose a 4-cell treatment (Table \ref{table:treatments}) in asking for bonus payment preferences. We will have a set of bonus payment options and will give participants a series of pairwise choices (example in Figure \ref{fig:choices}) to choose between. 
\begin{itemize}
  \item Group A will be asked for their stated preferences, but the actual bonus payment will be based on their votes.
  \item Group B will be asked for their preferences, and they will be told that one of their votes will be implemented at random to ensure incentive compatibility (random serial dictatorship model). 
  \item Within Group A and Group B, half of each group will receive some information on how each bonus structure would actually map to the distribution of work done by all participants (e.g. min, max, and mean payment amounts)
  \item The other half of Group A and Group B will receive only the bonus structure to vote on. (See Figure \ref{fig:choices} for an example.)
\end{itemize}

\begin{table}[ht]
    \centering
    \scriptsize
    \begin{tabular}{c|cc}
                            & No distribution info  &  Some distribution info \\
                            \hline
        Hypothetical voting & Group A1 ($n=250$)          &  Group A2 ($n=250$) \\
        Actual voting      & Group B1 ($n=250$)          &   Group B2 ($n=250$) 
    \end{tabular}
    \caption{An overview of treatment groups.}
    \label{table:treatments}
\end{table}

\begin{figure}[ht]
    \centering
    \includegraphics[width=8cm]
    {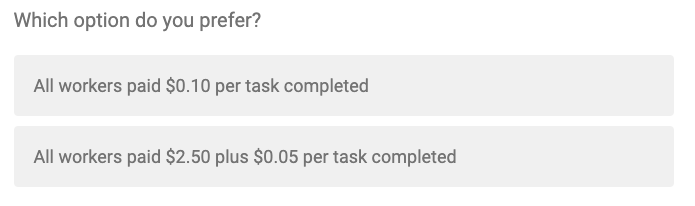}
    \caption{An example pairwise choice presented to participants}
    \label{fig:choices}
\end{figure}

\section{Planned Analysis}
\label{section:analysis}
Ideally, we would like to know people's preferences regarding model deployment impacts {\em before} the model gets deployed. However, prior to the effects of deployment being felt, the best we can do is gather {\em hypothetical} preferences regarding if the model {\em were to be deployed} in a certain manner. Although such hypothetical responses are easy to collect, they are only noisy estimates of real (post-deployment) preferences \citep[see, e.g.,][ regarding privacy preferences]{tan2018comparing}.

Our approach is therefore to use estimation techniques to predict preferences concerning new policies before they are actually implemented. In the case of our experiment, we will have collected hypothetical and actual preferences so can also compare the estimates with actual data.

Predicting actual preferences from hypothetical preferences will involve two main steps \cite{bernheim2013non}.
First, we train a regression model which predicts real bonus allocation choices---under a ``no information'' condition---from hypothetical bonus allocation choices (under the same condition).
Then, we use this model to predict real choices under the ``some information'' condition, using hypothetical choices under the same condition

\section{Discussion and Future Work}
As discussed in Section~\ref{section:fairpay}, there are strong structural correspondences between pay/work and ML. Thus the experiment proposed in Section~\ref{section:experiment} enables us to indirectly investigate preferences concerning the latter using concrete and easily comprehensible questions about preferences regarding the former, using the analyses in Section~\ref{section:analysis}. However, fairness is highly contextual, and we are careful not to claim that fairness concerns from one domain can simply be ``ported'' to another. Further domain-specific work, both pre- and post-deployment, will always be required when extrapolating stated preferences to actual preferences in different ML application domains.

Our approach aims to aggregate divergent interpretations of constructs like ``fairness'', but more work is required to bridge this line of preference-based work with approaches to fairness estimation based on techniques from measurement modeling \cite{jacobs2019measurement}. Beyond measurement, there is also the question of how to best combine stakeholder preferences in a broad ecosystem. Consider the stakeholders in a typical ecosystem in which ML is deployed, such as a social media platform: content creators, content consumers, advertisers, platform owners, etc. Each of these groups has different objectives and considerations, including different lenses on what is relevant to fairness. We might consider, for example, fairness for different groups of creators, for different groups of consumers, for different groups of advertisers, etc. And although much has been written about the incompatibilities of different technical definitions of ``fairness'' \cite{kleinberg2016inherent, chouldechova2017fair, verma2018}, little has been written about the incompatibilities of perspectives of the multitude of stakeholder roles. We propose examining the incentives for one particular group of stakeholders in this experiment, but in the future would seek to connect this to the broader system of stakeholders in a typical ML ecosystem. 

Lastly, there is also a value-laden question of how to incorporate different stakeholders into the design and creation of ML systems by specific entities with different value systems, which is another area of future work that our proposed experiment can help inform.


\bibliography{FairnessPreferences}

\begin{thebibliography}{26}
\providecommand{\natexlab}[1]{#1}
\providecommand{\url}[1]{\texttt{#1}}
\expandafter\ifx\csname urlstyle\endcsname\relax
  \providecommand{\doi}[1]{doi: #1}\else
  \providecommand{\doi}{doi: \begingroup \urlstyle{rm}\Url}\fi

\bibitem[Amodei \& Hernandez(2018)Amodei and Hernandez]{amodei31ai}
Amodei, D. and Hernandez, D.
\newblock {AI} and compute.
\newblock \emph{URL https://blog. openai. com/ai-and-compute}, 31, 2018.

\bibitem[Andreoni et~al.(2018)Andreoni, Aydin, Barton, Bernheim, and
  Naecker]{andreoni2018}
Andreoni, J., Aydin, D., Barton, B., Bernheim, B.~D., and Naecker, J.
\newblock When fair isn't fair: Understanding choice reversals involving social
  preferences.
\newblock Technical report, National Bureau of Economic Research, 2018.

\bibitem[Arrow(1951)]{arrow1951social}
Arrow, K.
\newblock Social choice and individual values, {Cowles Foundation Res. Econ.
  Yale Univ.}
\newblock In \emph{Cowles Commission Monograph}, volume~12. Wiley, NY and
  Chapman and Hall Los Alamitos, CA, 1951.

\bibitem[Bernheim et~al.(2013)Bernheim, Bjorkegren, Naecker, and
  Rangel]{bernheim2013non}
Bernheim, B.~D., Bjorkegren, D., Naecker, J., and Rangel, A.
\newblock Non-choice evaluations predict behavioral responses to changes in
  economic conditions.
\newblock Technical report, National Bureau of Economic Research, 2013.

\bibitem[Binns et~al.(2018)Binns, Van~Kleek, Veale, Lyngs, Zhao, and
  Shadbolt]{binns2018s}
Binns, R., Van~Kleek, M., Veale, M., Lyngs, U., Zhao, J., and Shadbolt, N.
\newblock `{It's} reducing a human being to a percentage' perceptions of
  justice in algorithmic decisions.
\newblock In \emph{Proceedings of the 2018 CHI Conference on Human Factors in
  Computing Systems}, pp.\  1--14, 2018.

\bibitem[Chouldechova(2017)]{chouldechova2017fair}
Chouldechova, A.
\newblock Fair prediction with disparate impact: A study of bias in recidivism
  prediction instruments.
\newblock \emph{Big data}, 5\penalty0 (2):\penalty0 153--163, 2017.

\bibitem[Crawford \& Joler(2018)Crawford and Joler]{crawford2018anatomy}
Crawford, K. and Joler, V.
\newblock Anatomy of ai, 2018.

\bibitem[Darlington(1971)]{darlington1971another}
Darlington, R.~B.
\newblock Another look at ``cultural fairness'' 1.
\newblock \emph{Journal of Educational Measurement}, 8\penalty0 (2):\penalty0
  71--82, 1971.

\bibitem[Gray \& Suri(2019)Gray and Suri]{gray2019ghost}
Gray, M.~L. and Suri, S.
\newblock \emph{Ghost Work: How to Stop Silicon Valley from Building a New
  Global Underclass}.
\newblock Eamon Dolan Books, 2019.

\bibitem[Grgic-Hlaca et~al.(2018)Grgic-Hlaca, Redmiles, Gummadi, and
  Weller]{grgic2018human}
Grgic-Hlaca, N., Redmiles, E.~M., Gummadi, K.~P., and Weller, A.
\newblock Human perceptions of fairness in algorithmic decision making: A case
  study of criminal risk prediction.
\newblock In \emph{Proceedings of the 2018 World Wide Web Conference}, pp.\
  903--912, 2018.

\bibitem[Harrison et~al.(2020)Harrison, Hanson, Jacinto, Ramirez, and
  Ur]{harrison2020empirical}
Harrison, G., Hanson, J., Jacinto, C., Ramirez, J., and Ur, B.
\newblock An empirical study on the perceived fairness of realistic, imperfect
  machine learning models.
\newblock In \emph{Proceedings of the 2020 Conference on Fairness,
  Accountability, and Transparency}, pp.\  392--402, 2020.

\bibitem[Jacobs \& Wallach(2019)Jacobs and Wallach]{jacobs2019measurement}
Jacobs, A.~Z. and Wallach, H.
\newblock Measurement and fairness.
\newblock \emph{arXiv preprint arXiv:1912.05511}, 2019.

\bibitem[Jacobs et~al.(2020)Jacobs, Blodgett, Barocas, Daum{\'e}~III, and
  Wallach]{jacobs2020meaning}
Jacobs, A.~Z., Blodgett, S.~L., Barocas, S., Daum{\'e}~III, H., and Wallach, H.
\newblock The meaning and measurement of bias: lessons from natural language
  processing.
\newblock In \emph{Proceedings of the 2020 Conference on Fairness,
  Accountability, and Transparency}, pp.\  706--706, 2020.

\bibitem[Jahanbakhsh et~al.(2017)Jahanbakhsh, Fu, Karahalios, Marinov, and
  Bailey]{jahanbakhsh2017you}
Jahanbakhsh, F., Fu, W.-T., Karahalios, K., Marinov, D., and Bailey, B.
\newblock You want me to work with who? stakeholder perceptions of automated
  team formation in project-based courses.
\newblock In \emph{Proceedings of the 2017 CHI Conference on Human Factors in
  Computing Systems}, pp.\  3201--3212, 2017.

\bibitem[Kleinberg et~al.(2016)Kleinberg, Mullainathan, and
  Raghavan]{kleinberg2016inherent}
Kleinberg, J., Mullainathan, S., and Raghavan, M.
\newblock Inherent trade-offs in the fair determination of risk scores.
\newblock \emph{arXiv preprint arXiv:1609.05807}, 2016.

\bibitem[Lee(2018)]{lee2018understanding}
Lee, M.~K.
\newblock Understanding perception of algorithmic decisions: Fairness, trust,
  and emotion in response to algorithmic management.
\newblock \emph{Big Data \& Society}, 5\penalty0 (1):\penalty0
  2053951718756684, 2018.

\bibitem[Pleiss et~al.(2017)Pleiss, Raghavan, Wu, Kleinberg, and
  Weinberger]{pleiss2017fairness}
Pleiss, G., Raghavan, M., Wu, F., Kleinberg, J., and Weinberger, K.~Q.
\newblock On fairness and calibration.
\newblock In \emph{Advances in Neural Information Processing Systems}, pp.\
  5680--5689, 2017.

\bibitem[Ross et~al.(2010)Ross, Irani, Silberman, Zaldivar, and
  Tomlinson]{ross2010crowdworkers}
Ross, J., Irani, L., Silberman, M.~S., Zaldivar, A., and Tomlinson, B.
\newblock Who are the crowdworkers? shifting demographics in {Mechanical Turk}.
\newblock In \emph{CHI'10 extended abstracts on Human factors in computing
  systems}, pp.\  2863--2872, 2010.

\bibitem[Samuelson(1948)]{samuelson1948consumption}
Samuelson, P.~A.
\newblock Consumption theory in terms of revealed preference.
\newblock \emph{Economica}, 15\penalty0 (60):\penalty0 243--253, 1948.

\bibitem[Saxena et~al.(2019)Saxena, Huang, DFlippis, Radanovic, Parkes, and
  Liu]{saxena2019}
Saxena, N.~A., Huang, K., DFlippis, E., Radanovic, G., Parkes, D.~C., and Liu,
  Y.
\newblock How do fairness definitions fare?: Examining public attitudes towards
  algorithmic definitions of fairness.
\newblock In \emph{Proceedings of the 2019 AAAI/ACM Conference on AI, Ethics,
  and Society}, pp.\  99--106, 2019.

\bibitem[Selbst et~al.(2019)Selbst, Boyd, Friedler, Venkatasubramanian, and
  Vertesi]{selbst2019fairness}
Selbst, A.~D., Boyd, D., Friedler, S.~A., Venkatasubramanian, S., and Vertesi,
  J.
\newblock Fairness and abstraction in sociotechnical systems.
\newblock In \emph{Proceedings of the Conference on Fairness, Accountability,
  and Transparency}, pp.\  59--68, 2019.

\bibitem[Srivastava et~al.(2019)Srivastava, Heidari, and
  Krause]{srivastava2019mathematical}
Srivastava, M., Heidari, H., and Krause, A.
\newblock Mathematical notions vs. human perception of fairness: A descriptive
  approach to fairness for machine learning.
\newblock In \emph{Proceedings of the 25th ACM SIGKDD International Conference
  on Knowledge Discovery \& Data Mining}, pp.\  2459--2468, 2019.

\bibitem[Tan et~al.(2018)Tan, Sharif, Bhagavatula, Beckerle, Mazurek, and
  Bauer]{tan2018comparing}
Tan, J., Sharif, M., Bhagavatula, S., Beckerle, M., Mazurek, M.~L., and Bauer,
  L.
\newblock Comparing hypothetical and realistic privacy valuations.
\newblock In \emph{Proceedings of the 2018 Workshop on Privacy in the
  Electronic Society}, pp.\  168--182, 2018.

\bibitem[Verma \& Rubin(2018)Verma and Rubin]{verma2018}
Verma, S. and Rubin, J.
\newblock Fairness definitions explained.
\newblock In \emph{Proceedings of the International Workshop on Software
  Fairness}, pp.\  1--7, 2018.

\bibitem[Woodruff et~al.(2018)Woodruff, Fox, Rousso-Schindler, and
  Warshaw]{woodruff2018qualitative}
Woodruff, A., Fox, S.~E., Rousso-Schindler, S., and Warshaw, J.
\newblock A qualitative exploration of perceptions of algorithmic fairness.
\newblock In \emph{Proceedings of the 2018 CHI Conference on Human Factors in
  Computing Systems}, pp.\  1--14, 2018.

\bibitem[Zafar et~al.(2017)Zafar, Valera, Rodriguez, Gummadi, and
  Weller]{zafar2017parity}
Zafar, M.~B., Valera, I., Rodriguez, M., Gummadi, K., and Weller, A.
\newblock From parity to preference-based notions of fairness in
  classification.
\newblock In \emph{Advances in Neural Information Processing Systems}, pp.\
  229--239, 2017.

\end{thebibliography}
\bibliographystyle{icml2020}

\end{document}